\documentclass[letterpaper]{article} % DO NOT CHANGE THIS
\usepackage{aaai25}  % DO NOT CHANGE THIS
\usepackage{times}  % DO NOT CHANGE THIS
\usepackage{helvet}  % DO NOT CHANGE THIS
\usepackage{courier}  % DO NOT CHANGE THIS
\usepackage[hyphens]{url}  % DO NOT CHANGE THIS
\usepackage{graphicx} % DO NOT CHANGE THIS
\usepackage{natbib}  % DO NOT CHANGE THIS AND DO NOT ADD ANY OPTIONS TO IT
\usepackage{caption} % DO NOT CHANGE THIS AND DO NOT ADD ANY OPTIONS TO IT
\usepackage{algorithm}
\usepackage{algorithmic}

\usepackage{array}

\usepackage{url}            % simple URL typesetting
\usepackage{booktabs}       % professional-quality tables
\usepackage{amsfonts}       % blackboard math symbols
\usepackage{nicefrac}       % compact symbols for 1/2, etc.
\usepackage{microtype}      % microtypography
\usepackage{xcolor}         % colors

\usepackage{subcaption}

\usepackage{amsmath}
\usepackage{amssymb}
\usepackage{amsthm}
\usepackage{mathtools}

\usepackage{multirow}

\newcommand{\px}{\mathbf{x}}

\newcommand{\pv}{\mathbf{v}}

\newcommand{\py}{\mathbf{y}}

\newcommand{\V}{\mathcal{V}}

\newcommand{\X}{\mathcal{X}}

\newcommand{\reals}{\mathbb{R}}

\newcommand{\expect}{\mathbb{E}}

\newtheorem{theorem}{Theorem}

\newtheorem{assumption}{Assumption}
\newtheorem{proposition}{Proposition}
\newtheorem{corollary}{Corollary}
\newtheorem{lemma}{Lemma}

\newtheorem{remark}{Remark}

\title{Error Bounds For Gaussian Process Regression Under Bounded Support Noise With Applications To Safety Certification}
\author {
    Robert Reed\textsuperscript{\rm 1},
    Luca Laurenti\textsuperscript{\rm 2},
    Morteza Lahijanian\textsuperscript{\rm 1}
}
\affiliations {
    \textsuperscript{\rm 1}Deptartment of Aerospace Engineering Sciences, University of Colorado Boulder, USA\\
    \textsuperscript{\rm 2}Delft Center for Systems and Control, Delft University of Technology, The Netherlands\\
    Robert.Reed-1@colorado.edu, L.Laurenti@tudelft.nl, Morteza.Lahijanian@colorado.edu
}

\begin{document}

\maketitle

\begin{abstract}
Gaussian Process Regression (GPR) is a powerful and elegant method for learning complex functions from noisy data with a wide range of applications, including in safety-critical domains. Such applications have two key features: (i) they require rigorous error quantification, and (ii) the noise is often bounded and non-Gaussian due to, e.g., physical constraints. While error bounds for applying GPR in the presence of non-Gaussian noise exist, they tend to be overly restrictive and conservative in practice. In this paper, we provide novel error bounds for GPR under bounded support noise. Specifically, by relying on concentration inequalities and assuming that the latent function has low complexity in the reproducing kernel Hilbert space (RKHS) corresponding to the GP kernel, we derive both probabilistic and deterministic bounds on the error of the GPR. We show that these errors are substantially tighter than existing state-of-the-art bounds and are particularly well-suited for GPR with neural network kernels, i.e., Deep Kernel Learning (DKL). Furthermore, motivated by applications in safety-critical domains, we illustrate how these bounds can be combined with stochastic barrier functions to successfully quantify the safety probability of an unknown dynamical system from finite data. We validate the efficacy of our approach through several benchmarks and comparisons against existing bounds. The results show that our bounds are consistently smaller, and that DKLs can produce error bounds tighter than sample noise, significantly improving the safety probability of control systems.
\end{abstract}

\section{Introduction}
\label{sec:intro}

Gaussian Process Regression (GPR) is a powerful and elegant method for learning complex functions from noisy data, renowned for its rigorous uncertainty quantification~\cite{rasmussen:book:2006}. This makes GPR particularly valuable in the control and analysis of \emph{safety-critical} systems \cite{berkenkamp2017safe,lederer2019local,lederer2019uniform,jagtap2020formal,Jackson2021,jackson2021formal,griffioen2023probabilistic, wajid2022formal}. However, in such systems, the underlying assumption of Gaussian measurement noise often does not hold. Measurements are typically filtered to reject outliers, and physical systems cannot traverse infinite distances in a single time step. Consequently, bounded support noise distributions provide a more accurate representation for many cyber-physical systems. But, without the Gaussian assumption, the mean and variance predictions of GPs cannot be directly used to represent confidence in the underlying system. To address this, recent works~\cite{hashimoto2022learning, Seeger, Chowdhury, jackson2021formal}
have diverged from a fully Bayesian approach and developed GPR error bounds under mild assumptions on the sample noise distribution, specifically sub-Gaussian and bounded noise (see Figure~\ref{fig:1d_gp_fnc}). While these bounds are useful, they tend to be overly restrictive and conservative, relying on parameters that must be over-approximated. This work aims to overcome these limitations by providing novel, tighter error bounds for GPR under bounded support noise, enhancing their applicability in safety-critical domains.

In this paper, we present novel error bounds for GPR under bounded support noise. Our key insight is that two factors contribute to the regression error: the error due to GP mean prediction at a point without noise, and the error due to noise perturbing the noise-free prediction by a factor proportional to the noise's support. Then, by relying on concentration inequalities and assuming the latent function has low complexity in the reproducing kernel Hilbert space (RKHS) corresponding to the GP kernel, we derive both probabilistic (i.e., with some confidence) and deterministic bounds on these error terms. Specifically, we use convex optimization and Hoeffding's Inequality to obtain accurate bounds. We demonstrate that these errors are substantially tighter than existing bounds, as depicted in Figure~\ref{fig:1d_gp_fnc}. Furthermore, our bounds are particularly well-suited for GPs with neural network kernels, such as Deep Kernel Learning (DKL). We also illustrate how these bounds can be combined with stochastic barrier functions~\cite{kushner1967stochastic,santoyo2021barrier}, which are a generalization of Lyapunov functions, to quantify the safety probability of an unknown control system from finite data. We validate our approach through several benchmarks and comparisons, showing that our bounds are consistently smaller than the state-of-the-art, resulting in more accurate safety certification of control systems.

The key contributions of this work are three-fold: 
(i) derivation of novel probabilistic and deterministic error bounds for GP regression with bounded support noise, (ii) demonstration of the application of these bounds with stochastic barrier functions to effectively quantify the safety probability of unknown control systems by solely using data, and (iii) validation of the approach through extensive benchmarks and comparisons, showing consistently tighter error bounds than state-of-the-art methods and enhanced accuracy in determining safety probabilities for barrier certificates.

\begin{figure}[t]
    \centering
    \begin{subfigure}[c]{0.23\textwidth}
        \includegraphics[width=\textwidth]{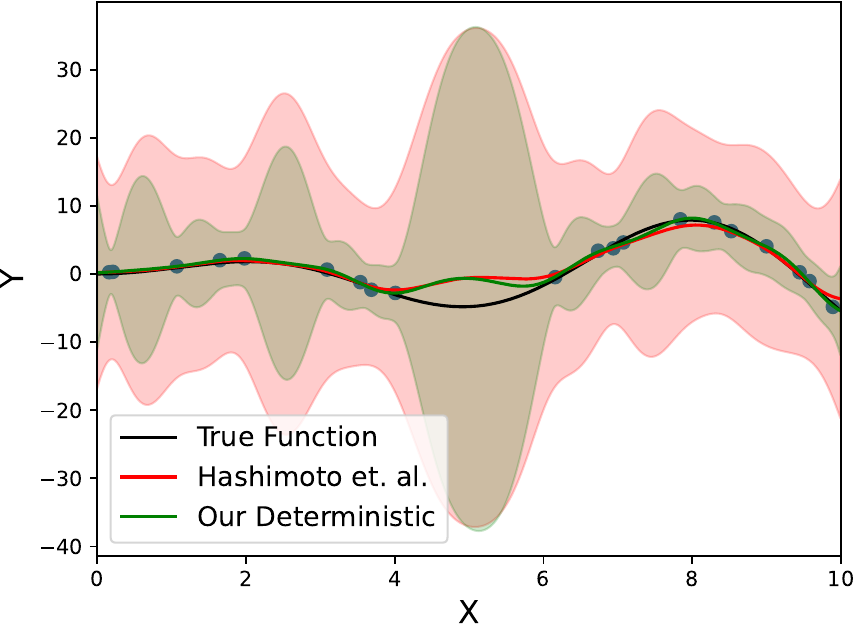}
        \caption{Deterministic Bounds} \label{fig:1d_GP_a}
    \end{subfigure}
    \begin{subfigure}[c]{0.23\textwidth}
        \includegraphics[width=\textwidth]{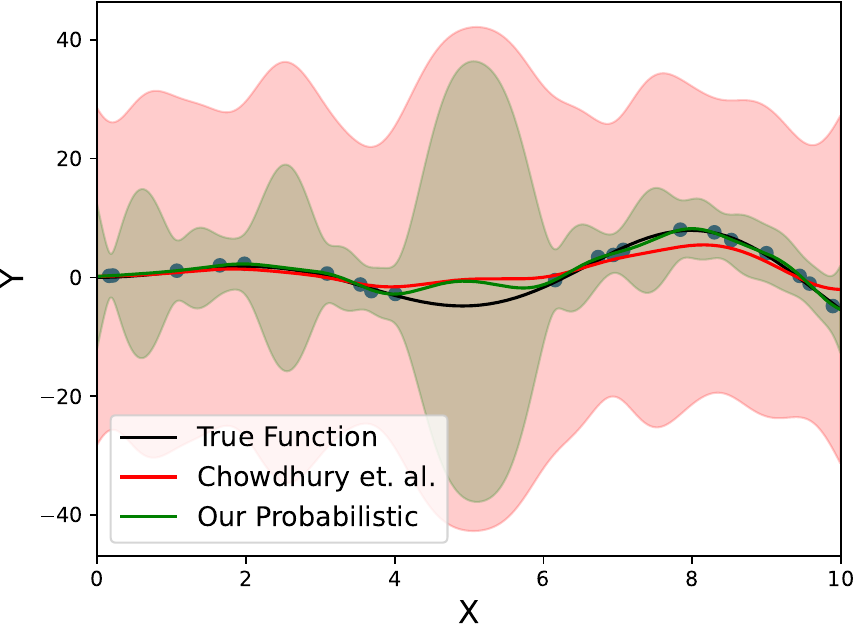}
        \caption{Probabilistic Bounds} \label{fig:1d_GP_b}
    \end{subfigure}
    \caption{Predictive mean and error bounds when learning from 20 samples of $y = x\sin(x) + v$ with $|v| < 0.5$. In (a) we plot the comparison of deterministic error bounds and mean predictions, and in (b) we show the probabilistic bounds that hold with 95\% probability. We set $\sigma_n = 0.1$ for our predictions, $\sigma_n = 0.5$ for Hashimoto et al. as per Lemma \ref{lemma:2}, and $\sigma_n^2 = 1 + 2/20$ for Chowdhury et al. as per Lemma \ref{lemma:1}. Note bounds for \citet{abbasi2013online} are nearly identical to Chowdhury et al. in this example, see Appendix \ref{sec:appendix_prob}. 
    }
    \label{fig:1d_gp_fnc}
\end{figure}

\paragraph{Related Works}
Several studies consider relaxations of the Gaussian assumption in the GP regression setting.  In particular,
\citeauthor{snelson2003warped} 
propose an approach that automatically generates a monotonic, non-linear warping function to map the observations into a latent space where it can be best modelled as a GP. However there is no guarantee that the warped data follows a GP.  Also, the method requires an inverse of the warping function, but it is generally only available in an approximate form. Hence, the model cannot be used to generate formal guarantees.
A closer work to ours is \cite{jensen2013bounded}, which derives posteriors for GP regression when the observation space has bounded support, and hence the noise has bounded support. However, similar to \cite{snelson2003warped}, the presented derivations provide approximate posteriors for the GP models, which limits their viability in safety-critical settings.

Works by \citeauthor{Seeger, abbasi2013online, Chowdhury, hashimoto2022learning} provide formal regression error bounds for GP models under non-Gaussian noise.
Specifically, work \cite{Seeger} first develops a framework for GP regression under the assumption that noise is $R$-sub-Gaussian and then formally quantifies probabilistic bounds on the GP prediction errors.  Using a similar framework, \cite{Chowdhury} derives tighter probabilistic bounds. Work \cite{abbasi2013online} derives similar bounds for Kernel Ridge Regression, which uses the same mean prediction as a GPR without a posterior variance.
While applicable, 
the setting in each of these works is more general than ours, which results in larger error terms when assessed with bounded support noise. Work \cite{hashimoto2022learning} specifically focuses on bounded noise and derives a deterministic error bound.  This bound is generally tighter than the prior probabilistic bounds; nevertheless, it becomes loose as the size of the support increases.  Since both error bounds in \cite{Chowdhury, hashimoto2022learning} are directly applicable to our scenario, we show their derivation in Section~\ref{sec:problem} and compare our results against them. In the experiments, we also compare against \cite{Seeger, abbasi2013online}. 

Finally, recent work \cite{maddalena2021deterministic} derives a tighter deterministic error bound when compared to the results of \cite{hashimoto2022learning} by using a Kernel Ridge Regression approach.  However, it requires the kernel to be strictly positive definite. While the squared exponential kernel, which is a popular choice in the controls application, can satisfy this requirement, the kernel matrix can be ill-conditioned resulting in inaccurate inversions. This restriction also limits the use of more expressive kernels that are positive semi-definite which may innately improve the error bounds, such as deep kernels. We show that our bounds are well-suited for deep kernels.

\section{Setup and Background}
\label{sec:problem}
We consider a stochastic system of the following form: 
\begin{align}
    \py = f(\px) + \pv, \label{true_dynamics}
\end{align}
where $\px \in \reals^d$ is the input, $\py \in \reals$ is the output, and $\pv \in V \subset \reals$ is an additive zero-mean\footnote{The assumption that $\expect[\pv] = 0$ is without loss of generality. In fact, if $\expect[\pv]\neq0$, we can add a bias term to $f$ to shift the expectation to 0.} random variable (noise) with bounded support $\sigma_v \in \reals_{\geq 0}$, i.e., $V = \{v \in \reals \mid |v| \leq \sigma_v \}$, and unknown distribution. Function $f: \reals^d \rightarrow \reals$ is \emph{unknown}.  
The main goal of the paper
is to regress $f$ from a dataset of input-output pairs $D = \{(x_i, y_i)\}_{i=1}^m$ and derive bounds on the error between the regressed model predictions and the true function $f$. However, taking $f$ as entirely unknown can lead to an ill-posed problem; hence, we impose a standard smoothness (well-behaved) assumption \cite{Seeger, Jackson2021} on $f$. 
\begin{assumption} [RKHS Continuity] \label{assumption:1}
For a compact set $X \subset \reals^d$, let $\kappa: \reals^d \times \reals^d \rightarrow \reals_{\geq 0}$ be a given continuously differentiable kernel and $\mathcal{H}_\kappa(X)$ the reproducing kernel Hilbert space (RKHS) corresponding to $\kappa$ over $X$ with induced norm $\|\cdot\|_\kappa$. Let $f(\cdot) \in \mathcal{H}_\kappa(X)$. Moreover, there exists a constant $B \geq 0$ s.t. $\|f(\cdot)\|_\kappa \leq B$.
\end{assumption}
Assumption \ref{assumption:1} implies that $f(x) = \sum_{n=1}^\infty \alpha_n \kappa(x, x_n)$ for representer points $\alpha_n \in \reals$ and $x_n \in \reals^d$. Note that for most kernels $\kappa$  used in practice, such as the squared exponential kernel, $\mathcal{H}_\kappa$ 
is dense in the space of continuous functions \cite{steinwart2001influence}.

We stress that, in this paper, we do not assume that $f$ is a sample from a Gaussian process prior. Moreover, since the noise is non-Gaussian, the likelihood model is also not Gaussian. Nevertheless, similar to the \emph{agnostic} setting described in \cite{Seeger,Chowdhury,hashimoto2022learning}, we would still like to use the GP regression framework to regress $f$ with misspecified noise and prior models. 
In the remainder of this section, we provide a brief background on GP regression and state the existing error bounds on the learning error.  
In Section \ref{sec:error_bound}, 
we derive new bounds  and discuss why they are tighter than the existing ones. 
We also illustrate the practical utility of these bounds in a control application in Section~\ref{sec:applications}. 
Finally, we provide empirical evaluations of the new errors bounds  in Section~\ref{sec:experiment}.

\subsection{Existing Error Bounds Using GPs and RKHS}

Let $D = \{(x_i, y_i)\}_{i=1}^m$ be a dataset consisting of $m$ input-output observations pairs from System~\eqref{true_dynamics}. A popular method to predict the output of $f$ at a new input $x^*$ with confidence on the predicted value is Gaussian Process regression (GPR). 
A Gaussian Process (GP) is a collection of random variables, such that any finite collection of those random variables is jointly Gaussian \cite{rasmussen:book:2006}. GPR starts by placing a Gaussian prior over $f$ through use of a kernel $\kappa: \reals^d \times \reals^d \rightarrow \reals_{\geq 0}$. Then,  under the assumption that $\pv$ is a zero-mean Gaussian independent random variable with variance $\sigma_n^2$ for any new input $x^*,$ the predictive mean $\mu_{D}(x^*)$ and variance $\sigma^2_{D}(x^*)$ can be derived as:
\begin{align}
    &\mu_{D}(x^*) = K_{x^*, \X}(K_{\X,\X} + \sigma_n^2 I)^{-1} Y, \label{mean_pred}\\
    &\sigma^2_{D}(x^*) = K_{x^*, x^*} - K_{x^*, \X}(K_{\X,\X} + \sigma_n^2 I)^{-1} K_{\X, x^*}, \label{covar_pred}
\end{align}
where $\X = [x_1, \ldots, x_{m}]^T$ and $Y = [y_1, \ldots, y_m]^T$ are respectively the vectors of input and output data, $K_{\X,\X}$ is a matrix whose $i$-th row and $j$-th column is $\kappa(x_i, x_j)$, and $K_{x,\X} = K_{\X, x}^T$ are vectors whose $i$-th entry is $\kappa(x, x_i)$. 

As already mentioned, in our setting in System \eqref{true_dynamics}, we do not assume $f$ is a sample from a GP and the noise $\pv$ is strictly bounded (hence, non-Gaussian). Fortunately, recent works show that, even under such settings, as long as Assumption~\ref{assumption:1} holds, the GP analytical posterior prediction can still be used outside of the Bayesian framework, i.e., even with misspecified noise and priors. 
Specifically, Work \cite{Chowdhury} shows that when the measurement noise is conditionally $R$-sub-Gaussian, which includes bounded support distributions, probabilistic bounds on the error between the prediction $\mu_D$ and $f$ can be derived as follows.
\begin{lemma}[{\cite[Theorem 2]{Chowdhury}}]
    Let $X \subset \reals^d$ be a compact set, $B > 0$ be the bound on $\|f\|_{\kappa} \leq B$, and $\Gamma \in \reals_{>0}$ be the maximum information gain of $\kappa$.
    If noise $\pv$ has a $R$-sub-Gaussian distribution and $\mu_{D}$ and $\sigma_{D}$ are obtained via Equations \eqref{mean_pred}-\eqref{covar_pred} on a dataset $D$ of size $m$ with $\sigma_n^2 = 1 + 2/m$, then it holds that, for every $\delta \in (0,1]$,
    \begin{align}
        \mathbb{P}\Big(\forall x \in X,\; |\mu_{D}(x) - f(x)| \leq \beta(\delta)\sigma_{D}(x) \Big) \geq 1 - \delta,
    \end{align}
    where $\beta(\delta) = B + R\sqrt{2(\Gamma + 1+ \log{(1/\delta}))}$.
    \label{lemma:1}
\end{lemma}
Lemma~\ref{lemma:1} provides probabilistic bounds on the GP regression error $|\mu_{D}(x) - f(x)|$
by relying on the kernel parameters such as RKHS norm bound $B$ and information gain $\Gamma$, which are often difficult to obtain accurately.  
They can however be (conservatively) bounded using techniques introduced in ~\cite{Seeger, jackson2021formal, hashimoto2022learning}.
Another important observation in Lemma~\ref{lemma:1} is that, when noise $\pv$ has a $R$-sub-Gaussian distribution, variance $\sigma_n^2$ is set to $1+2/m$, which adds conservatism to the bounds by substantially increasing the variance of GP predictions. 
Note that to obtain estimates with confidence $1$, i.e., $\delta = 0,$ Lemma \ref{lemma:1} requires infinitely-many samples. 
Alternative probabilistic bounds in this setting are considered by \citeauthor{Seeger}, which makes use of similar parameters, and by \citeauthor{abbasi2013online} which leaves $\sigma_n$ as a decision variable 
(see Appendix \ref{sec:appendix_prob}).
As shown in our experiments, they are also conservative.

In the case that noise $\pv$ has bounded support (i.e., $|\pv| \leq \sigma_v$), deterministic bounds on the prediction error are provided by \citeauthor{hashimoto2022learning}.
\begin{lemma}[{\cite[Lemma 2]{hashimoto2022learning}}] 
    Let $X$, $\kappa$, and $B$ be as in Lemma~\ref{lemma:1}.
    If noise $\pv \in V$ has a finite support $\sigma_v$ (i.e., $|\pv| \leq \sigma_v$) and $\mu_{D}$ and $\sigma_{D}$ are obtained via Equations \eqref{mean_pred}-\eqref{covar_pred} on a dataset $D$ of size $m$ with $\sigma_n = \sigma_v$, then it holds that, for every $ x \in X$,
    \begin{align}
        |\mu_{D}(x) - f(x)| \leq \beta_{T}\sigma_{D}(x),
    \end{align}
    where $\beta_{T} = \sqrt{B^2 - Y^T(K_{\X,\X} + \sigma_v^2I)^{-1}Y + m}$, and $Y$ and $K_{\X,\X}$ are as in~Equations \eqref{mean_pred}-\eqref{covar_pred}.
    \label{lemma:2}
\end{lemma}
By restricting noise to a bounded set, Lemma~\ref{lemma:2} is able to provide a deterministic bound on the GP regression error $|\mu_{D}(x) - f(x)|$. Unlike the probabilistic error bounds in Lemma~\ref{lemma:1}, the deterministic bound in Lemma~\ref{lemma:2} does not rely on the information gain $\Gamma$, removing a source of conservatism. However, it is generally conservative when $\sigma_v$ is not small, due to placing $\sigma_n=\sigma_v$ in Equations~\eqref{mean_pred}-\eqref{covar_pred}. 

Note that
the bounds in Lemma~\ref{lemma:1} employ assumptions that are too general for our problem, 
i.e., $R$-sub-Gaussian (conditioned on the filtration) vs bounded support noise independent of filtration, 
and rely on parameters that must be approximated. Similarly, the bounds in Lemma~\ref{lemma:2} do not allow probabilistic reasoning and are restrictive when the support on noise is large. Hence, there is a need for 
probabilistic error bounds derived specifically for bounded support noise, as well as a need for an improved deterministic error bound that does not grow conservatively with the size of the support.

\section{Bounded Support Error Bounds}
\label{sec:error_bound}

In this section, we introduce novel probabilistic and deterministic error bounds for GPR for System~\eqref{true_dynamics}, where noise has a bounded support distribution. We show that in this setting, the results of Lemmas~\ref{lemma:1} and \ref{lemma:2} can be substantially improved.

\subsection{Probabilistic Error Bounds}
We begin with the probabilistic bounds. With an abuse of notation, for the vector of input samples $\X = [x_1, \ldots, x_m]^T$, let $f(\X) = [f(x_1), \ldots, f(x_m)]^T$. Then, we observe that the noise output vector $Y$ is such that $Y=f(\X) + \V$, where $\V = [v_1, \ldots, v_m]^T$ is a vector of i.i.d. samples of the noise, each of which is bounded by $\sigma_v$. 
Consequently, from Eqn~\eqref{mean_pred} and by denoting $G = (K_{\X,\X} + \sigma_n^2I)^{-1}$ and $W_x = K_{x,\X}G$, we can bound the GP regression learning error as
\begin{align}
    |\mu_{D}(x) - f(x)| &= |W_x(f(\X) + \V) - f(x)| \\  % |W_xY - f(x)| =
    & \leq |W_xf(\X) - f(x)| + |W_x \V|.
    \label{Eqn:InequalityProof}
\end{align}
Hence, the error is bounded by a sum of two terms: $|W_xf(\X) - f(x)|$, which is the error due to mean prediction at $x$ with no noise, and $|W_x \V|$, which is the error due to the noise with a value at most proportional to $\sigma_v$.
The following lemma, which extends results in \cite{hashimoto2022learning}, bounds the first term.
\begin{lemma} 
    \label{lemma:3} 
    Let $X$, $\kappa$, $B$, and $D$ be as in Lemma~\ref{lemma:1}, and $G = (K_{\X,\X} + \sigma_n^2I)^{-1}$, $W_x = K_{x,\X}G$, and $c^* \leq f(\X)^T G f(\X)$. Then, it holds that, for every $x \in X$,
    \begin{align}
        |W_xf(\X) - f(x)| \leq \sigma_{D}(x)\sqrt{B^2 - c^*}.
    \end{align} 
\end{lemma}
The proof of Lemma \ref{lemma:3} can be found in Appendix \ref{App:Lemma_proof}.
In Theorem~\ref{my_prob_theorem}, we rely on Hoeffding's Inequality~\cite{mohri2018foundations} to bound the second term in Eqn~\eqref{Eqn:InequalityProof}, which in turn provides a probabilistic bound on the GP regression error when combined with Lemma \ref{lemma:3}.

\begin{theorem}
    [Bounded Support Probabilistic RKHS Error]
    Let $X$, $\kappa$, $B$, $D$, $G$, $W_x$, and $c^*$ be as in Lemma~\ref{lemma:3}, and define
    $\lambda_x = 4\sigma_v^2K_{x,\X}G^2 K_{\X,x}.$
    If noise $\pv \in V$ is zero-mean and has a finite support $\sigma_v$ (i.e., $|\pv| \leq \sigma_v$) and $\mu_{D}$ and $\sigma_{D}$ are obtained via Eqns~\eqref{mean_pred}-\eqref{covar_pred} on dataset $D$ with any choice of $\sigma_n >0$, then it holds that, for every $ x \in X$ and $\delta \in (0, 1]$,
    \begin{align}
        \mathbb{P}\Big(|\mu_{D}(x) - f(x)| \leq \epsilon(x,\delta)\Big) \geq 1 - \delta,
    \end{align}
    where $\epsilon(x,\delta) = \sigma_{D}(x)\sqrt{B^2 - c^*} + \sqrt{\frac{\lambda_x}{2}\text{ln}\frac{2}{\delta}}$.
    \label{my_prob_theorem}
\end{theorem}
The proof uses Lemma~\ref{lemma:3} for the first term of $\epsilon(x,\delta)$ and derives
the second term by applying Hoeffding's Inequality to $|W_x \V|$ by noting that $\expect[W_x \V] = 0$ and each $-\sigma_v \leq v_i \leq +\sigma_v$, enabling the random variable to maintain bounded support on each term. The full proof can be found in Appendix \ref{App:theorem_proof}.
Note that Theorem \ref{my_prob_theorem} only requires two parameters in its probabilistic bound: $c^*$ and $B$.
In fact, $c^*$ can be found  by solving the following quadratic optimization problem, where we rely on the boundedness of the support of the distribution of $\mathcal{V} = [v_1, \ldots, v_m]^T$:
\begin{align}
  c^* =  \min_{-\sigma_v \leq v_i \leq \sigma_v, i=1,...,m} (Y - \mathcal{V})^T G (Y - \mathcal{V}).
\end{align}
The other parameter, 
$B$, can be formally bounded by the technique introduced in \cite{jackson2021formal}. This is in contrast with Lemma \ref{lemma:1}, which also requires an approximation on the information gain of the kernel. We should also emphasize that Theorem \ref{my_prob_theorem} allows $\sigma_n$ to remain as a decision variable, enabling an optimization over $\sigma_n$ that can further minimize the error bounds as compared to Lemmas \ref{lemma:1} and \ref{lemma:2}. 

We also note as $m\rightarrow\infty$ then $\lambda_x$ tends toward 0 which implies that we can set $\delta$ arbitrarily close to 0, reducing the error bound to the result of Lemma~\ref{lemma:3} which decreases with samples as $\sigma_D(x)$ decreases and $c^*$ increases. This demonstrates an $O(\sqrt{m})$ improvement over the results of Lemma~\ref{lemma:2}. A detailed discussion is provided in Appendix \ref{app:disscuss}.

We extend our point-wise probabilistic error bounds to a uniform bound with the following corollary.
\begin{corollary} [Uniform Error Bounds]
For a given compact set $X' \subseteq X$,
    \begin{align}
    \mathbb{P}\Big(|\mu_{D}(x) - f(x)| \leq \overline{\epsilon}_{X'}(\delta)\Big) \geq 1 - \delta
\end{align}
$\forall x \in X'$, where $\overline{\epsilon}_{X'}(\delta) = \sup_{x\in X'}{\epsilon}(\delta,x)$. \label{cor:uniform_bound}
\end{corollary}

\subsection{Deterministic Error Bounds}
If error bounds with confidence one ($\delta=0$) are desired, Theorem \ref{my_prob_theorem} would require infinite samples. Here, we show how confidence one results can be derived with finite samples, producing an alternative deterministic error bound to the one in Lemma~\ref{lemma:2}.

\begin{theorem}
    [Bounded Support Deterministic RKHS Error]
    Let $X$, $B$, $D$, $G$, $c^*$ be as in Theorem~\ref{my_prob_theorem}, and $\Lambda_x = \sum_{i=1}^m \sigma_v|[K_{x,\X}G]_i|$. 
    If noise $\pv \in V$ has a finite support $\sigma_v$ (i.e., $|\pv| \leq \sigma_v$) and $\mu_{D}$ and $\sigma_{D}$ are obtained via Eqns~\eqref{mean_pred}-\eqref{covar_pred} on dataset $D$ with any choice of $\sigma_n >0$, then it holds that, for every $ x \in X$,
    \begin{align}
        |\mu_{D}(x) - f(x)| \leq \epsilon_d(x)
    \end{align}
    where $\epsilon_d(x) =\sigma_{D}(x)\sqrt{B^2 - c^*} + \Lambda_x$.
    \label{my_det_theorem}
\end{theorem}
\begin{proof}
    Consider the term $|W_x \V|$ in Eqn~\eqref{Eqn:InequalityProof}.
    The noise distribution that maximizes $|W_x \V|$ is found by setting $v_i = \mathrm{sign}([W_x]_i)\sigma_v$, where $\mathrm{sign}([W_x]_i) = 1$ if the $i$-th term of $W_x$ is $\geq 0$ and 0 otherwise. Using this bound and Lemma~\ref{lemma:3} on the RHS of Eqn~\eqref{Eqn:InequalityProof}, we conclude the proof.
\end{proof}

\begin{remark}
    We stress that while Lemmas \ref{lemma:1} and \ref{lemma:2} define strict requirements on the value of $\sigma_n$ used in the posterior prediction, i.e., $1 + 2/m$ and $\sigma_v$ respectively, both Theorems \ref{my_prob_theorem} and \ref{my_det_theorem} allow for any value to be used for $\sigma_n$. In particular, $\sigma_n \ll \sigma_v$ is a valid selection which can reduce the predictive variance, enabling much tighter deterministic and probabilistic bounds. Similarly, when $\sigma_v$ is very small, we can choose $\sigma_n > \sigma_v$ to avoid numerical instabilities in inverting $(K_{\X,\X} + \sigma_n^2I)$. 
\end{remark}

\begin{figure*}[ht]
    \centering
    \begin{subfigure}[c]{0.242\textwidth}
        \includegraphics[width=\textwidth]{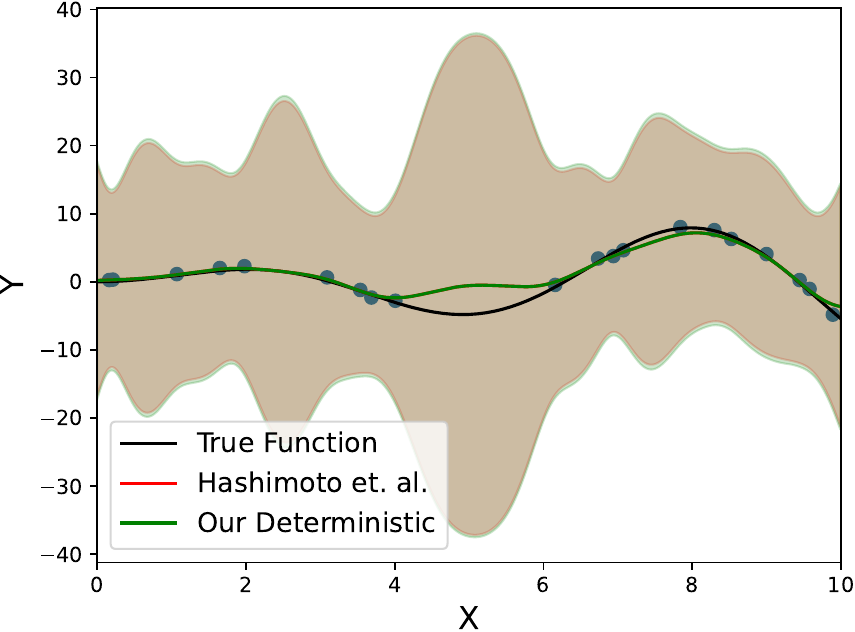}
        \caption{$\sigma_n = \sigma_v$} \label{fig:1d_GP_vary_a}
    \end{subfigure}
    \begin{subfigure}[c]{0.242\textwidth}
        \includegraphics[width=\textwidth]{figures/Det_fnc_bounds_5.0-crop.pdf}
        \caption{$\sigma_n = \sigma_v/5$} \label{fig:1d_GP_vary_b}
    \end{subfigure}
    \begin{subfigure}[c]{0.242\textwidth}
        \includegraphics[width=\textwidth]{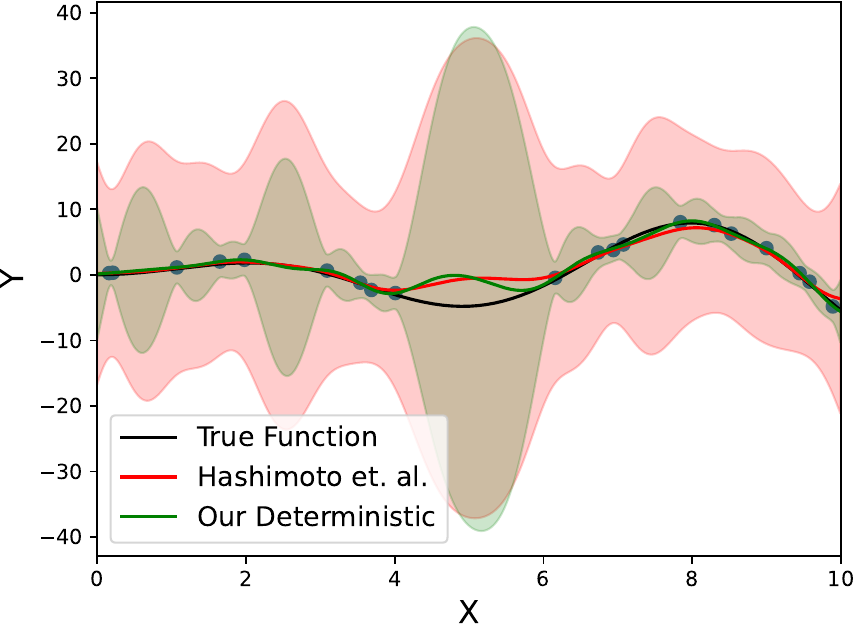}
        \caption{$\sigma_n = \sigma_v/10$} \label{fig:1d_GP_vary_c}
    \end{subfigure}
    \begin{subfigure}[c]{0.242\textwidth}
        \includegraphics[width=\textwidth]{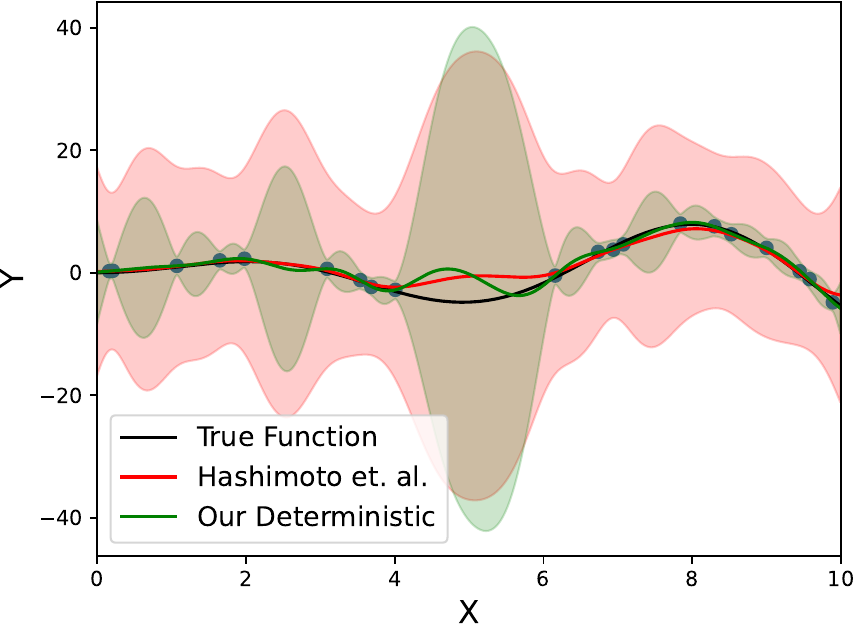}
        \caption{$\sigma_n = \sigma_v/20$} \label{fig:1d_GP_vary_v}
    \end{subfigure}
    \caption{
    Predictive mean and deterministic error bounds when learning from 20 samples of $y = x\sin(x) + v$ with $|v| < 0.5$ as $\sigma_n$ varies. We plot the mean and bounds from Lemma \eqref{lemma:2} in red, with our mean and bounds in green. In all cases our bounds remain valid, but demonstrate optimal performance when $\sigma_n \in [\sigma_v/10, \sigma_v/5]$.}
    \label{fig:1d_gp_fnc_var_sigman}
\end{figure*}

We demonstrate the trend of the error bounds as we adjust $\sigma_n$ used for the posterior predictions of a GP in Figure~\ref{fig:1d_gp_fnc_var_sigman}. It is immediately clear that using a smaller $\sigma_n$ when there is more data results in tight error bounds across the domain. 

\begin{remark}
    In this paper, we consider the offline setting, where we are given a set of i.i.d. data.  In this setting, the bounds proposed in \cite{abbasi2013online, Seeger, Chowdhury} are still valid. To extend to online setting, our bounds can be updated as data is gathered, and the noise must remain independent of the filtration, which is typical in robotics applications.
\end{remark}

\subsection{Extension to Deep Kernel Learning.}
Deep Kernel Learning (DKL) is an extension of GPR~\cite{Ober2021}.
In DKL, we compose a base kernel, e.g., the squared-exponential $\kappa_{se}= \sigma_s \exp\left(-\|x-x'\| / 2l^2 \right)$, with a neural network as $\kappa_{DKL}(x, x') =\kappa_{se}(\psi_w(x), \psi_w(x'))$, where $\psi_w: \reals^d \rightarrow \reals^s$ is a neural network parameterized by weights $w$. Then posterior predictions still use analytical Eqns~\eqref{mean_pred}-\eqref{covar_pred}, but the kernel now includes significantly more parameters which has been shown to significantly improve the representational power of GPs without needing more data in the kernel~\cite{reed2023promises}. 

\begin{remark}
    \label{DKL_validity}
    For DKL to satisfy Assumption~\ref{assumption:1}, the kernel must remain continuously differentiable and positive semi-definite. Using the $\mathrm{GeLU}$ or $\mathrm{Tanh}$ activation functions and learning $\psi_w(x)$ as a model of $f(x)$ using stochastic mini-batching can prove sufficient. 
    Then, under the assumption that $\psi_w$ is well-behaved 
    over a compact set $X$, the RKHS norm $\|f\|_{\kappa_{DKL}} \leq \|f\|_\kappa$ is feasible. This can be inferred directly from \cite[Proposition 2]{jackson2021formal}, as DKL tends to correlate data more effectively over space, it is reasonable to expect that $\inf_{x,x'\in X}\kappa_{DKL}(x,x') \geq \inf_{x,x'\in X}\kappa(x,x')$. 
\end{remark}

DKL can reduce the posterior variance of a GP and use the same RKHS norms as the base kernel, i.e., decreasing $\sigma_D$ without increasing bound $B$.
Therefore, DKL can enable tighter GP regression error bounds with the same number of predictive samples than standard GP.
Similarly, in the event that many samples are available, network $\psi_w$ can be pre-trained over a large set of samples and the kernel can use a small subset of the data for posterior predictions, enabling computationally efficient calculations with increased accuracy. %\LL{Extend this? }
Furthermore, compared to Lemma~\ref{lemma:1}, our error bounds utilize significantly more information about the kernel. This allows an informed kernel, such as DKL, to further reduce the bound beyond just the value of $\sigma_D$.

\section{Application to Safety of Control Systems}
\label{sec:applications}
The bounds we derive in Theorems \ref{my_prob_theorem} and \ref{my_det_theorem} can be particularly important to provide safety guarantees for dynamical systems learned from data. 
Specifically, consider the following model of a discrete-time stochastic system with additive noise
\begin{align}
\label{eqn:SysDynamics}
    \px(k+1) = f(\px(k)) + \pv \qquad k\in  \mathbb{N},
\end{align}
where $\px \in \reals^d$ is the state,  $\pv \in V \subset \reals^d$ is a random variable independent and identically distributed at each time  step representing an additive noise term, and $f:\mathbb{R}^d \to \mathbb{R}^d$ is an unknown, possibly non-linear, function (vector field). Intuitively, $\px(k)$ represents a general model of a stochastic system taking values in $\mathbb{R}^d$.

A key challenge in many applications is to guarantee that System \eqref{eqn:SysDynamics} will stay within a given safe set $X_s \subset \reals^d$, i.e., it avoids obstacles, for a given time horizon $[0,N]$. A common technique to obtain such guarantees is to rely on stochastic barrier functions \cite{kushner1967stochastic, santoyo2021barrier}, which represents a generalization of Lyapunov functions to provide set invariance. The intuition behind barriers is to build a function $\mathcal{B}:\mathbb{R}^d \to \mathbb{R}$ that when composed with System \eqref{eqn:SysDynamics} forms a supermartingale and consequently, martingale inequalities can be used to study the systems' evolution, as shown in the following Proposition.  
\begin{proposition}[ \cite{mazouz2022safety}, Section 3.1]
\label{prop:barrier}
Given an initial set $X_0$, a safe set $X_s$, and an unsafe set $X_u = \reals^d \setminus X_s $, a twice differentiable function $\mathcal{B}:\reals^d \rightarrow \reals$ is a barrier function if the following conditions hold for  $\beta,\eta \in [0,1]$:
$\mathcal{B}(x) \geq 0$ for all $x\in \reals^d$, $\mathcal{B}(x) \leq \eta$ for all $x\in X_0$, $\mathcal{B}(x) \geq 1$ for all $\forall x\in X_u$, and 
\begin{align}
    \expect[\mathcal{B}(f(x) + \pv)\mid x] \leq \mathcal{B}(x) + \beta \quad \forall x\in X_s. 
    \label{eq: barrier martingale}
\end{align}
Then, it holds that $$P\big(\forall k \in [0,N], \px(k)\in X_s\mid \px(0)\in X_0 \big) \geq 1 - (\beta N + \eta).$$
\end{proposition}
Hence, by finding an appropriate $\mathcal{B}$, safety of System~\eqref{eqn:SysDynamics} can be guaranteed.
However, how to construct such $\mathcal{B}$ when $f$ is unknown is still an open problem \cite{jagtap2020formal,wajid2022formal,mazouz2022safety,mazouz2024piecewise}. 
Here, we show that, under the  assumption that $f$  lies in the RKHS of $\kappa$ and $\|f^{(i)}\|_\kappa < B_i$, where $f^{(i)}$ denotes the $i$-th output of $f$, we can employ GPR along with Theorems~\ref{my_prob_theorem} and \ref{my_det_theorem} to construct $\mathcal{B}$.

In particular, by partitioning $X_s$ into a set of convex regions $Q = \{q_1,...,q_{|Q|}\}$ s.t. $X_s = \cup_{q\in Q} q$, it is sufficient to simply replace the condition in~\eqref{eq: barrier martingale} with the following constraints 
for each $q\in Q$:
\begin{subequations}
    \begin{align}
        &\expect[\mathcal{B}(z_q + \pv)\mid x \in q] \leq \mathcal{B}(x)+\beta  && \\
        &z^{(i)}_q \in [\mu_{D}^{(i)}(x) \pm \overline{\epsilon}_q^{(i)}(\delta)] && \forall x \in q,
    \end{align}
\end{subequations}
where $\mu_{D}^{(i)}$ denotes the mean prediction for $f^{(i)}$ and $\overline{\epsilon}_q^{(i)}(\delta)$ can be computed by combining Corollary~\ref{cor:uniform_bound} with either Theorem~\ref{my_prob_theorem} or~\ref{my_det_theorem}. Work \cite{mazouz2024piecewise} shows how to calculate a barrier with such constraints.
In the former case, the safety probability in Proposition \ref{prop:barrier}  holds with the confidence resulting from Theorem~\ref{my_prob_theorem}, while in the latter case results hold with confidence $1$. In Section~\ref{sec:experiment}, we consider both cases.

\section{Experiments}
\label{sec:experiment}
\begin{figure*}[h]
    \centering
    \begin{subfigure}[c]{0.23\textwidth}
        \includegraphics[width=\textwidth]{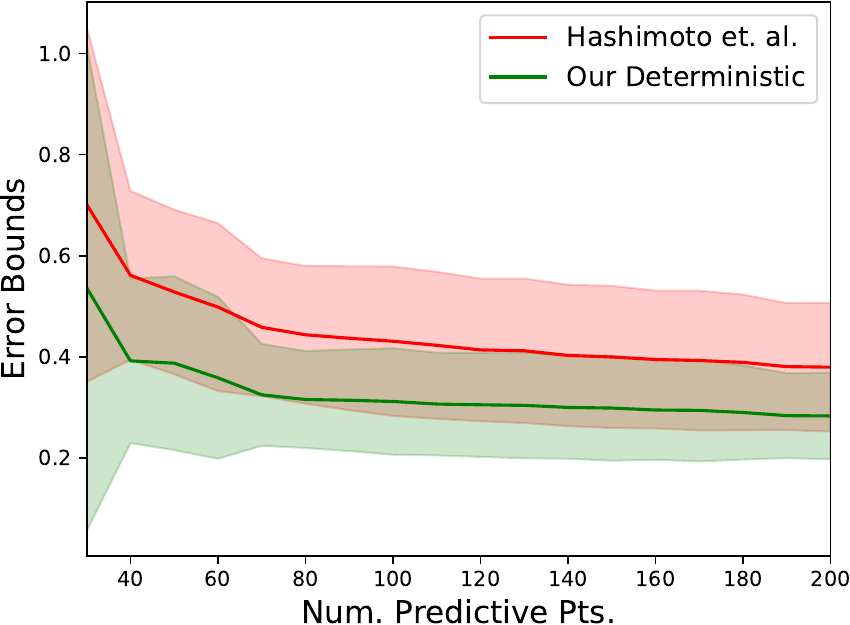}
        \caption{GP - det. error comparison} \label{fig:2d_GP_a}
    \end{subfigure}
    ~~
    \begin{subfigure}[c]{0.23\textwidth}
        \includegraphics[width=\textwidth]{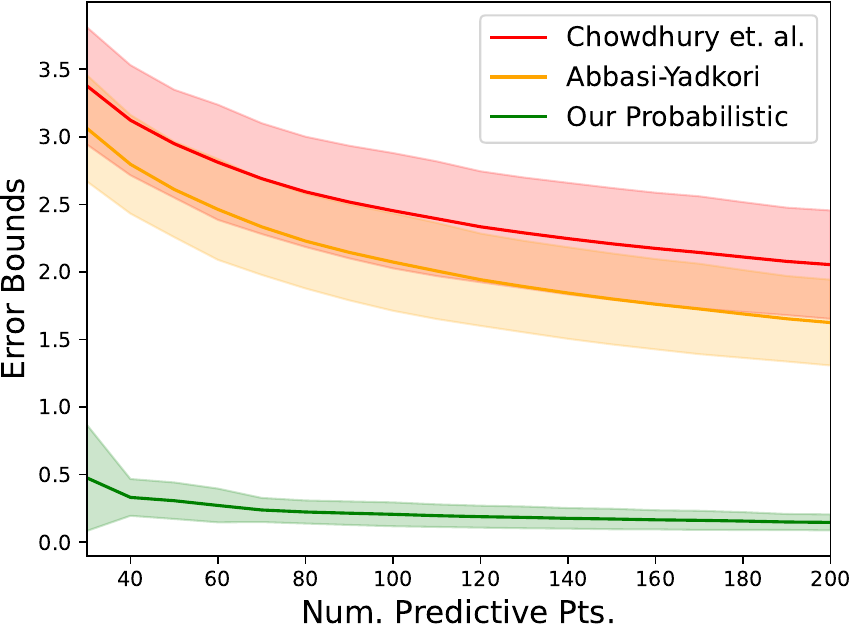}
        \caption{GP - prob. error comparison} \label{fig:2d_GP_b}
    \end{subfigure}
    ~~
    \begin{subfigure}[c]{0.23\textwidth}
        \includegraphics[width=\textwidth]{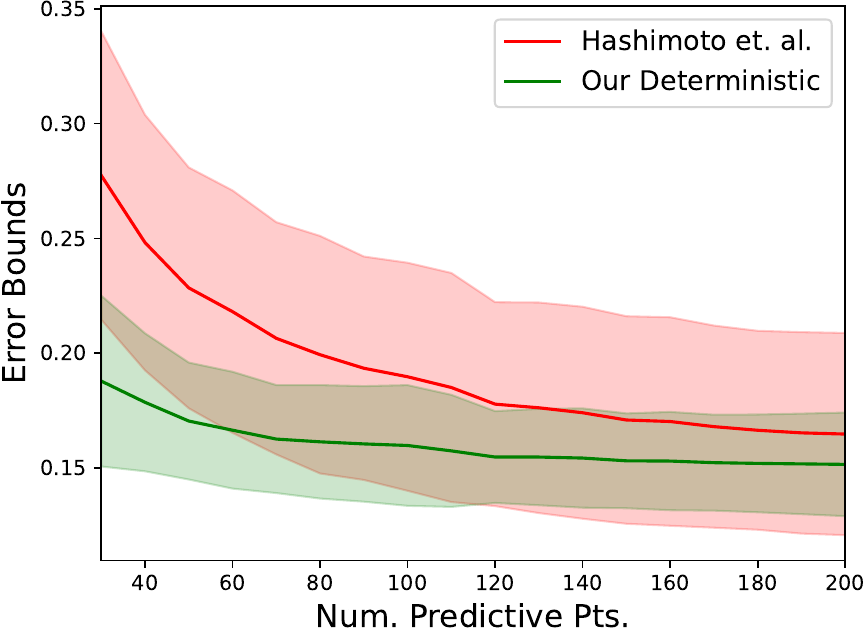}
        \caption{DKL - det. error comparison} \label{fig:2d_DKL_a}
    \end{subfigure}
    ~~
    \begin{subfigure}[c]{0.23\textwidth}
        \includegraphics[width=\textwidth]{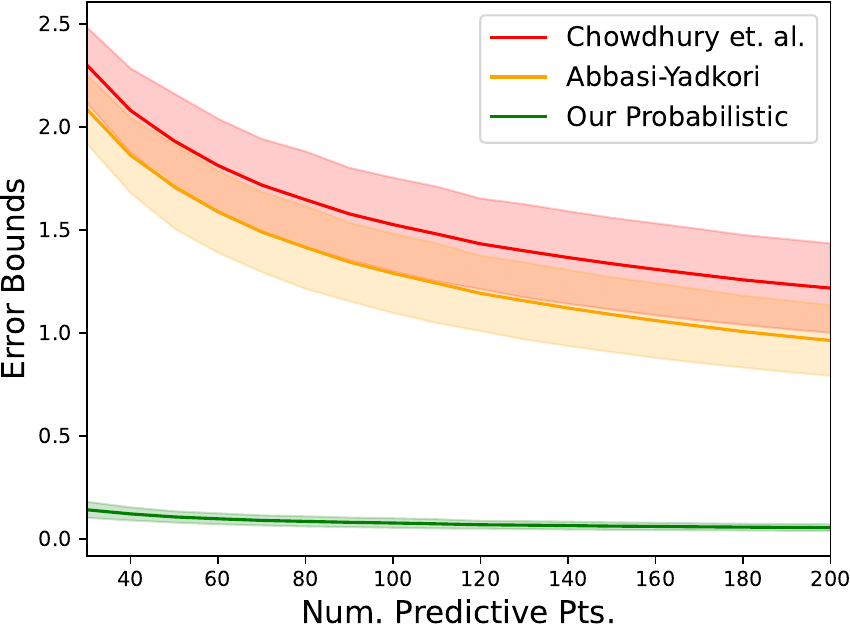}
        \caption{DKL - prob. error comparison} \label{fig:2d_DKL_b}
    \end{subfigure}
    \caption{
    Trends of error bounds with increasing data used for the posterior prediction for a 2D system when $|v|\leq 0.1$ using Left: the \emph{squared exponential kernel} and Right: \emph{DKL}. In (a) and (c) we compare our deterministic error bound (green) with results from Lemma~\ref{lemma:2} (red). In (b) and (d) we compare our probabilistic error bound (green) to results from Lemma~\ref{lemma:1} (red) and \citeauthor{abbasi2013online} (orange) with $\delta = 0.05$. In (c) and (f) we compare our deterministic and probabilistic bounds with $\delta\in [0.01, 0.5]$.}
    \label{fig:2d_gp_bounds}
\end{figure*}
\begin{figure}[!h]
    \centering
    \begin{subfigure}[c]{0.23\textwidth}
        \includegraphics[width=\textwidth]{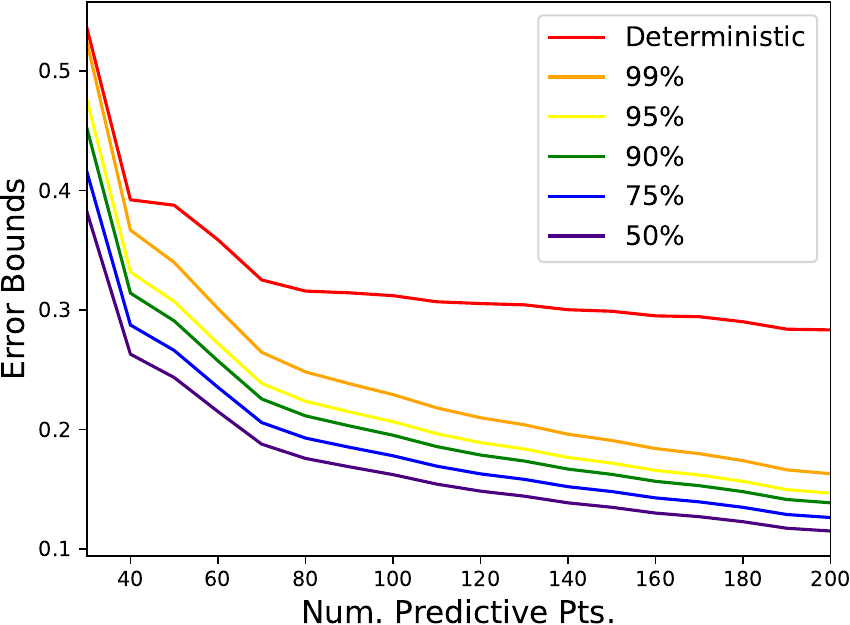}
        \caption{GP - vary $\delta$ comparison} \label{fig:2d_GP_c}
    \end{subfigure}
    \begin{subfigure}[c]{0.23\textwidth}
        \includegraphics[width=\textwidth]{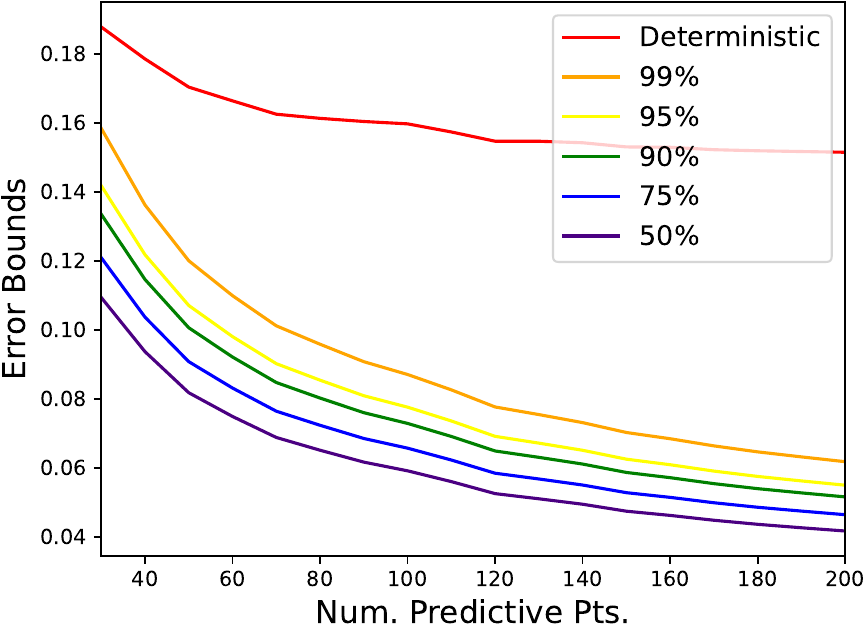}
        \caption{DKL - vary $\delta$ comparison} \label{fig:2d_DKL_c}
    \end{subfigure}
    \caption{
    Trends of our error bounds with increasing data used for the posterior prediction for a 2D system when $|v|\leq 0.1$ using (a) the \emph{SE kernel} and (b) \emph{DKL}. We compare our deterministic (Theorem \ref{my_det_theorem}) and probabilistic (Theorem \ref{my_prob_theorem}) bounds with $\delta\in [0.01, 0.5]$.}
    \label{fig:2d_gp_bounds_vary}
\end{figure}

In this section, we demonstrate the effectiveness of our bounds in comparison to the existing state-of-the-art bounds. We first provide a visual of the trends of our bounds as more data is utilized for posterior predictions, as well as demonstrating the improvements that DKL can provide over standard GPs. We then demonstrate how our bounds perform in comparison to the state-of-the-art across several dynamical systems of interest. Finally, we show how our bounds can be used in safety certification, through the use of stochastic barrier functions, as described in Section \ref{sec:applications}. We note that deep kernels in general are only positive semi-definite, which prohibits the use of techniques in \cite{maddalena2021deterministic} for generating a deterministic bound. Further detail on the models considered can be seen in Appendix \ref{App:norms_domains}.

\textbf{2D Visual Example}
We consider the case of GP regression for a linear system with $\px \in \reals^2$ and $|v| \leq 0.1$, defined as
\begin{align}
    \py = 0.4x_1 + 0.1x_2 + v.
\end{align}
We first consider a squared exponential kernel and using the methods proposed in \cite{hashimoto2022learning} we set $B = 10$. For our predictions, we set $\sigma_n = \sigma_v/5 = 0.02$. We illustrate the trend of the bounds as the number of data points in the kernel increases in Figures~\ref{fig:2d_GP_a}-\ref{fig:2d_GP_b}, showing the averaged mean bound and one standard deviation over $10^4$ test points.
We note that our deterministic bound is comparable or better than the existing approach and our probabilistic bound significantly outperforms existing bounds. 

Next, we consider the same scenario but with a DKL kernel. Results  are shown in Figures \ref{fig:2d_DKL_a}-\ref{fig:2d_DKL_b}. Here, we assume a fixed neural network that is pre-trained to model $f(x)$ from data. We consider a network with two hidden layers of 16 neurons each with the GeLU activation function. The network is trained over 1000 sample points through stochastic mini-batching. Figure~\ref{fig:2d_gp_bounds_vary} compares averages of our deterministic bound to our probabilistic bound. Interestingly, DKL generates probabilistic errors lower than the bound on sample noise, enabling highly accurate predictions even with very noisy samples. We note that, for each model, our probabilistic bound remain accurate for the desired confidence level.

\textbf{Examples for Multiple Dynamical Systems}
\begin{table*}[!ht]
    \centering
    \setlength{\tabcolsep}{1mm}
    \begin{tabular} { l l @{\hspace{3mm}} l @{\hspace{5mm}} | @{\hspace{5mm}} c  c @{\hspace{3mm}}  @{\hspace{3mm}} c  c @{\hspace{5mm}} | @{\hspace{5mm}} c c @{\hspace{3mm}}  @{\hspace{3mm}} c c @{\hspace{3mm}}  @{\hspace{3mm}} c c @{\hspace{3mm}}  @{\hspace{3mm}} cc }
        \toprule
        \multirow{2}{*}{System \;}& \multirow{2}{*}{$\kappa$\;} & \multirow{2}{*}{$\sigma_n$} &  \multicolumn{2}{l}{\!\!\underline{\> Our Det. \>}} & \multicolumn{2}{l}{\!\!\underline{Lem \ref{lemma:2} Det.} } & \multicolumn{2}{l}{\!\!\underline{ Our Prob. } } & \multicolumn{2}{l}{\!\!\underline{ AY Prob. }} & \multicolumn{2}{l}{\!\!\underline{Lem \ref{lemma:1} Prob.}} & \multicolumn{2}{l}{\!\!\underline{SKKS Prob.}} \\
          & & & true & $\|\epsilon\|_1$ & true & $\|\epsilon\|_1$  &  true & $\|\epsilon\|_1$ & true & $\|\epsilon\|_1$ & true & $\|\epsilon\|_1$ & true & $\|\epsilon\|_1$  \\
        \hline
        2D Lin & SE & $\sigma_v/5$ & $0.06$ & $0.86$ & $0.04$ & $1.45$ & $0.06$ & $0.57$ &  0.10 &  5.72 & $0.10$ & 7.53 & $0.04$ & $350$ \\
        2D Lin & SE & $\sigma_v/10$ & $0.07$ & $0.84$ & $0.04$ & $1.45$ & $0.07$ & $0.50$ & 0.10 & 5.72 & $0.10$ & 7.53 & $0.04$ & $350$\\
        2D Lin & DKL & $\sigma_v/5$ & $0.04$ & $0.36$ & $0.04$ & $0.45$ & $0.04$ & $0.15$ & 0.08 &  2.72 & $0.08$ & $3.66$ & $0.04$ & $141$\\
        2D Lin & DKL & $\sigma_v/10$ & $0.04$ & \textbf{0.33} & $0.04$ & $0.45$ & $0.04$ & \textbf{0.12} &  0.08 &  2.72 & $0.08$ & $3.66$ & $0.04$ & $141$ \\
        \hline
        2D NL & SE & $\sigma_v/5$ & 0.09 & 1.23 & 0.08 & 1.68 & 0.09 & 0.94 &  0.32 &  7.47 & 0.32 & 9.29 & 0.08 & 350 \\
        2D NL & SE & $\sigma_v/10$ & 0.11 & 1.31 & 0.08 & 1.68 & 0.11 & 0.90 &  0.32 &  7.47 & 0.32 & 9.29 & 0.08 & 350 \\
        2D NL & DKL & $\sigma_v/5$ & 0.04 & 0.39 & 0.04 & 0.53 & 0.04 & 0.22 & 0.19 &  3.79  & 0.19 & 4.75 & 0.04 & 130 \\
        2D NL & DKL & $\sigma_v/10$ & 0.04 & \textbf{0.35} & 0.04 & 0.53 & 0.04 & \textbf{0.18} &  0.19  &  3.79 & 0.19 & 4.75 & 0.04 & 130 \\
        \hline
        3D DC & SE & $\sigma_v/5$ & 0.10 & 1.96 & 0.09 & 2.34 & 0.10 & 1.13 &  0.41 & 9.61 & 0.41 & 13.67 & 0.09 & 1071  \\
        3D DC & SE & $\sigma_v/10$ & 0.12 & 2.13& 0.09 & 2.34 & 0.12 & 1.06 & 0.41 &  9.61 & 0.41 & 13.67 & 0.09 & 1071 \\
        3D DC & DKL & $\sigma_v/5$ & 0.03 & 0.44 & 0.03 & 0.49 & 0.03 & 0.12 &  0.29 &  2.82 & 0.29 & 4.32 & 0.03 & 195 \\
        3D DC & DKL & $\sigma_v/10$ & 0.03 & \textbf{0.43} & 0.03 & 0.49 & 0.03 & \textbf{0.11} &  0.29 & 2.82 & 0.29 & 4.32 & 0.03 & 195 \\
        \hline
        5D Car & SE & $\sigma_v/5$ & 0.51 & 10.33 & 0.30 & 11.7 & 0.51 & 5.60 &  0.36  &  14.0 & 0.36 & 22.63  & 0.30 & 3475\\
        5D Car & SE & $\sigma_v/10$ & 0.68 & 13.0 & 0.30 & 11.7 & 0.68 & 6.16 &  0.36 & 14.0  & 0.36 & 22.63 & 0.30 & 3475 \\
        5D Car & DKL & $\sigma_v/5$ & 0.06 & 1.48 & 0.06 & 1.54 & 0.06 & 0.46 & 0.25 & 4.34 & 0.25 & 6.82 & 0.06 & 581 \\
        5D Car & DKL & $\sigma_v/10$ & 0.06 & \textbf{1.43} & 0.06 & 1.54 & 0.06 & \textbf{0.40} & 0.25 & 4.34 & 0.25 & 6.82 & 0.06 & 581\\
        \bottomrule
    \end{tabular}
    \caption{
    Average $L_1$ error bounds ($\|\epsilon\|_1$) over 10000 test points. 
    We report the value for various values of $\sigma_n$ and for the squared exponential kernel (SE) and for DKL. We report both the true error ($\|\mu - f\|_1$) induced by the model estimated empirically over $10^4$ samples and the bounds produced by Theorem~\ref{my_det_theorem} (Our Det.), Lemma~\ref{lemma:2} (Lem 2 Det.), and probabilistic bounds from Theorem~\ref{my_prob_theorem} (Our Prob.), \citeauthor{abbasi2013online} (AY Prob.), Lemma~\ref{lemma:1} (Lem 1 Prob.), and \citeauthor{Seeger} (SKKS Prob.) in order setting $\delta = 0.05$. Lemma~\ref{lemma:2} and \cite{Seeger} set $\sigma_n = \sigma_v$ and Lemma~\ref{lemma:1} and \cite{abbasi2013online} set $\sigma_n^2 = 1 + 2/m$. 
    }
    \label{table:predictive_results_with_truth}
\end{table*}
\begin{table*}[!h]
    \centering
    \begin{tabular} { c | c c c c | c c c c | c c c c}
        \toprule
        \multirow{2}{*}{Model} & \multicolumn{4}{c}{\underline{\quad Our Prob. \quad}}  & \multicolumn{4}{c}{\underline{\quad Our Det. \quad}} & \multicolumn{4}{c}{\underline{\quad Lem \ref{lemma:2} Det. \quad}} \\
        & $\eta$ & $\beta$ & $P_s$ & $t$ & $\eta$ & $\beta$ & $P_s$ & $t$ & $\eta$ & $\beta$ & $P_s$ & $t$ \\
        \hline
        Pendulum & 1e-6 & 1e-6 & 0.999 & 136 & 1e-6 & 0.077 & 0.923 & 147 & 1e-6 & 0.499 & 0.499 & 2979 \\
        % 3D Cartpole & - & - & - & - & - & - & - & - & - \\
        4D Linear & 1e-6 & 1e-6 & 0.999 & 451 & 1e-6 & 0.172 & 0.827 & 2880 & 1e-6 & 0.249 & 0.749 & 13233 \\
        \bottomrule
    \end{tabular}
    \caption{Barrier results, where $t$ is time in seconds taken to synthesize a barrier, $P_s = 1 - (\eta + N\beta)$, and $N=1$. 
    } 
    \label{table:Piecewise}
\end{table*}
We compare our bounds to existing bounds (Lemma \ref{lemma:1}-\ref{lemma:2}) and bounds proposed by~\citeauthor{abbasi2013online, Seeger} for a variety of dynamical systems from literature \cite{Jackson2021, Adams:CSL:2022, reed2023promises}, namely a 2D linear model, a 2D non-linear model, a 3D Dubin's car model, and a 5D second order car model. We consider identical noise distributions in all dimensions; for the 2D and 3D systems we consider $\pv\sim \text{Uniform}(-0.1, 0.1)$ and for the 5D system $\pv\sim \text{Uniform}(-0.2, 0.2)$. Results are in Table \ref{table:predictive_results_with_truth}. 

We note that modifying the value of $\sigma_n$ used for posterior predictions does not impact the accuracy of the model significantly but can be optimized to improve the error bounds. In all cases, we see that DKL (with a well trained neural network prior) significantly outperforms standard GP models when computing error bounds. We emphasize that the probabilistic bounds from \cite{abbasi2013online}, Lemma~\ref{lemma:1}, and \cite{Seeger} are significantly larger than ours for two primary reasons: (i) those bounds generalize to any conditional sub-Gaussian distribution while ours is specific to bounded support, and (ii) our bound incorporates significantly more information about the kernel, allowing an informed kernel to reduce the bound further than just the value of $\sigma_D(x)$. This results in our probabilistic bound being at least an order of magnitude smaller than prior works.

\textbf{Safety Certification of Unknown Stochastic Systems}
We consider the Inverted Pendulum (2D) agent from the Gymnasium environment and a contractive linear 4D system for data-driven safety certification using the formulation in Section~\ref{sec:applications}. We collected data for the Pendulum model under the best controller available from the OpenAI Leaderboard \cite{gymLeader} and perturb the systems with PERT distributions (a transformation of a Beta distribution with a specified mean). We construct a barrier using the Dual-Approach suggested in \cite{mazouz2024piecewise} and compare results when generating interval bounds on $f$ with Corollary~\ref{cor:uniform_bound} using Theorem~\ref{my_det_theorem}, Lemma~\ref{lemma:2}, and Theorem~\ref{my_prob_theorem} with $\delta = 0.05$ using DKL models.

The barriers identify a lower bound on the probability of the system remaining in a predefined safe set for a given horizon when initialized with $\theta, \dot{\theta} \in [-0.025, 0.025]\times[-0.055, 0.055]$ for the Pendulum and each state in $[-0.1, 0.1]$ for the 4D system. Barriers are synthesized on an Intel Core i7-12700K CPU at 3.60GHz with 32 GB of RAM. Results are reported in Table \ref{table:Piecewise}. Safety probabilities using bounds of Lemma \ref{lemma:1} and \cite{abbasi2013online} are not reported in the table because they result in $P_s = 0$ after three time steps for both models.

Barrier certificates using Theorem \ref{my_prob_theorem} are based on a 95\% confidence. We see that our bounds allow for a significant improvement in certification results as compared to existing bounds. For instance, for a time horizon $N=10$, applying Proposition~\ref{prop:barrier} with our bounds results in guarantees of $99.9$\% safety probability with 95\% confidence, whereas using Lemma~\ref{lemma:2} results in a $0$\% safety probability for the Pendulum model. We see similar results for the 4D system, where the deterministic bounds remain too conservative to identify the contractive nature of the system yet our probabilistic bounds enable guarantees of safety with high confidence. We also note that the reduced conservatism of our approach enable barrier synthesis to be significantly faster.

\section{Conclusion}
\label{sec:conclusion}
In this paper, we derive novel error bounds (both probabilistic and deterministic) for GP regression under bounded support noise. We demonstrate that by assuming sample noise has zero mean, error bounds that are tighter than sample noise can be achieved. 
We show that our error bounds utilize significantly more information about the kernel than existing probabilistic bounds (i.e. each term informed by $K_{\X,x}$); hence, they are well-suited for regression techniques based on informed kernels such as DKL that correlate $x$ with the entire dataset.  
We also show
the improvements that our bounds can provide over the state-of-the-art in safety certification of control systems through the use of stochastic barrier functions, generating certificates with significantly larger safety probabilities.
Future directions include application of these bounds in safe control and shield design in reinforcement learning.

\newpage
\section{Acknowledgments}
This work was supported by the Air Force Research Lab (AFRL) under agreement number FA9453-22-2-0050.

\bibliography{references}

\newpage

\appendix
\section{Appendix: Alternate RKHS Bounds}
\label{sec:appendix_prob}
\subsection{Alternate Probabilistic Bounds}
The following are the probabilistic bounds proposed by \cite{abbasi2013online}, written in a format consistent with the notation of GPR rather than Kernel Ridge Regression.
\begin{lemma}[{\cite[Theorem 3.11]{abbasi2013online}}]
    Let $X \subset \reals^d$ be a compact set, $B > 0$ be the bound on $\|f\|_{\kappa} \leq B$, and $L \in \reals_{>0}$ be the maximum value of $\kappa(\cdot,\cdot)$ over $X$.
    If noise $\pv$ has a conditional $R$-sub-Gaussian distribution and $\mu_{D}$ and $\sigma_{D}$ are obtained via Equations \eqref{mean_pred}-\eqref{covar_pred} on a dataset $D$ of size $m$ with any choice of $\sigma_n > 0$, then it holds that, for every $\delta \in (0,1]$,
    \begin{align}
        \mathbb{P}\left(\forall x \in X,\; |\mu_{D}(x) - f(x)| \leq \sigma_{D}(x) \Delta(\delta) \right) \geq 1 - \delta,
    \end{align}
    where $\Delta(\delta) = \left( \frac{R}{\sigma_n} \sqrt {2 \log \left(\frac{1}{\delta}\right)  + \frac{(m - 1)L^2}{\sigma_n^2} } + B \right)$.
    \label{lemma:thesis}
\end{lemma}

We note that the value of $\sigma_n$ is a decision variable here. We set $\sigma_n^2 = 1 + 2/m$ as in the bounds for \cite{Chowdhury} as these produced the best results with Lemma \ref{lemma:thesis} as shown in Table \ref{table:thesis_pred_results}. We also show the similarity of these bounds with the results of Lemma \ref{lemma:1} in Figure \ref{fig:1d_gp_appendix}, which compares the bounds for the example in Figure \ref{fig:1d_GP_b}.

\begin{table}[h]
    \centering
    \begin{tabular} { l c c | c  c }
        \toprule
        \multirow{2}{*}{System}& \multirow{2}{*}{$\kappa$}& \multirow{2}{*}{$\sigma_n$} & \multicolumn{2}{c}{\underline{\quad Lem \ref{lemma:thesis} Prob. \quad}}\\
          & & & true & $\|\epsilon\|_1$  \\
        \midrule
        2D Lin & SE & $\sigma_v/5$ & 0.06 & 44.7  \\
        2D Lin & SE & $\sigma_v/10$ & 0.07 & 101  \\
        2D Lin & SE & $\sqrt{(1 + 2/m)}$ & 0.10 & 5.72\\
        2D Lin & DKL & $\sigma_v/5$ & 0.04 & 45.5  \\
        2D Lin & DKL & $\sigma_v/10$ & 0.04 & 21.5 \\
        2D Lin & DKL & $\sqrt{(1 + 2/m)}$ & 0.08 &  2.72 \\
        \midrule
        2D NL & SE & $\sigma_v/5$ & 0.04 & 68.8 \\
        2D NL & SE & $\sigma_v/10$  & 0.11 & 178 \\
        2D NL & SE & $\sqrt{(1 + 2/m)}$  & 0.32 &  7.47 \\
        2D NL & DKL & $\sigma_v/5$ & 0.04  & 23.2  \\
        2D NL & DKL & $\sigma_v/10$ & 0.04  & 48.2  \\
        2D NL & DKL & $\sqrt{(1 + 2/m)}$ & 0.19  &  3.79 \\
        \midrule
        3D DC & SE & $\sigma_v/5$  & 0.10 & 171  \\
        3D DC & SE & $\sigma_v/10$ & 0.10 & 429  \\
        3D DC & SE & $\sqrt{(1 + 2/m)}$ & 0.41 & 9.61  \\
        3D DC & DKL & $\sigma_v/5$ & 0.03  & 26.7 \\
        3D DC & DKL & $\sigma_v/10$ & 0.03  & 55.7 \\
        3D DC & DKL & $\sqrt{(1 + 2/m)}$ &  0.29 & 2.82  \\
        \midrule
        5D Car & SE & $\sigma_v/5$  & 0.51 & 409 \\
        5D Car & SE & $\sigma_v/10$ & 0.68 & 1373 \\
        5D Car & SE & $\sqrt{(1 + 2/m)}$ &  0.36 &  14.0 \\
        5D Car & DKL & $\sigma_v/5$  & 0.06  & 41.5 \\
        5D Car & DKL & $\sigma_v/10$  & 0.06  & 85.5 \\
        5D Car & DKL & $\sqrt{(1 + 2/m)}$  & 0.25 & 4.34 \\
        \bottomrule
    \end{tabular}
    \caption{
    Average $L_1$ error bounds ($\|\epsilon\|_1$) over 10000 test points. 
    We report the value for various values of $\sigma_n$ and for the squared exponential kernel (SE) and for DKL. We report both the true error ($\|\mu - f\|_1$) induced by the model estimated empirically over 10000 samples and the bounds from Lemma \ref{lemma:thesis} \cite{abbasi2013online} setting $\delta = 0.05$. 
    }
    \label{table:thesis_pred_results}
\end{table}

\begin{figure}[h]
    \centering
    \begin{subfigure}[c]{0.23\textwidth}
        \includegraphics[width=\textwidth]{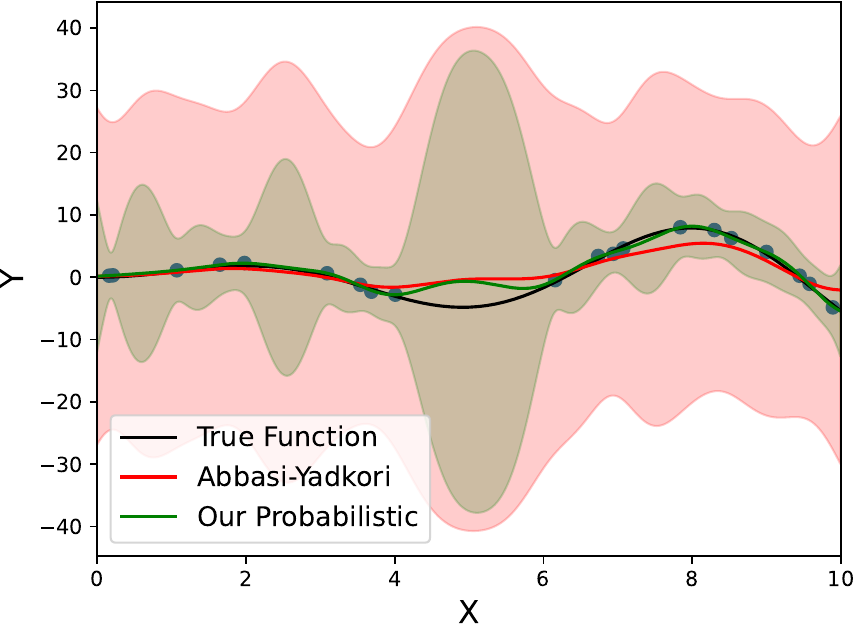}
        \caption{Abbasi-Yadkori Bounds} \label{fig:1d_GP_a_appendix}
    \end{subfigure}
    \begin{subfigure}[c]{0.23\textwidth}
        \includegraphics[width=\textwidth]{figures/prob_chowdhury_fnc_bounds_5.0-crop.pdf}
        \caption{Chowdhury Bounds} \label{fig:1d_GP_b_appendix}
    \end{subfigure}
    \caption{Predictive mean and probabilistic error bounds for $\delta = 0.05$ when learning from 20 samples of $y = x\sin(x) + v$ with $|v| < 0.5$. In (a) we plot the comparison with \cite{abbasi2013online}, and in (b) we show the comparison with \cite{Chowdhury} We set $\sigma_n = 0.1$ for our predictions and $\sigma_n^2 = 1 + 2/20$ for \cite{Chowdhury} as per Lemma \ref{lemma:1} and for \cite{abbasi2013online}
    }
    \label{fig:1d_gp_appendix}
\end{figure}

The bounds proposed in \cite{Seeger} are as follows.
\begin{lemma} [{\cite[Theorem 6]{Seeger}}]
    Let $X \subset \reals^d$ be a compact set, $B > 0$ be the bound on $\|f\|_{\kappa} \leq B$, and $\Gamma \in \reals_{\geq0}$ be the maximum information gain of the kernel $\kappa$ with $m$ data points.
    If noise $\pv$ has a conditional $R$-sub-Gaussian distribution and $\mu_{D}$ and $\sigma_{D}$ are obtained via Equations \eqref{mean_pred}-\eqref{covar_pred} on a dataset $D$ of size $m$ with any choice of $\sigma_n = \sigma_v$, then it holds that, for every $\delta \in (0,1]$,
    \begin{align}
        \mathbb{P}\left(\forall x \in X,\; |\mu_{D}(x) - f(x)| \leq \sigma_{D}(x) \xi(\delta) \right) \geq 1 - \delta,
    \end{align}
    where $\xi(\delta) = \sqrt {2B^2 + 300\Gamma \log^3(m/\delta) }$.
    \label{lemma:seeger}
\end{lemma}

\subsection{Alternate Deterministic Error Bounds}
The bounds proposed in \cite{maddalena2021deterministic} are as follows, written in a format consistent with GPR rather than Kernel Ridge Regression (recall that this work requires the kernel to be strictly positive definite). We note that this bound is larger than the bound derived in Theorem \ref{my_det_theorem} in general.

\begin{lemma}[{\cite[Theorem 1]{maddalena2021deterministic}}]
    Let $K_{\X,\X}$ be a positive definite kernel matrix. Then with $m$ data points in $\X$, with observations $Y$, and $\sigma_n > 0$, then
    \begin{multline}
    |\mu_{D}(x) - f(x)| \leq P(x)\sqrt{B^2 + \Delta - Y^TK_{\X,\X}^{-1}Y} + \\
    \tilde{\Lambda}_x + |Y^T ( K_{\X,\X} + \frac{1}{m\sigma_n}K_{\X,\X}^2)^{-1}K_{\X,x}| \label{maddalena_bound}
\end{multline}
where 
$$P(x) = \sqrt{\kappa(x,x) - K_{x, \X}K_{\X,\X}^{-1} K_{\X, x}},$$
$$\Delta = \max_{-\sigma_v \leq v_i \leq \sigma_v, i=1,...,m}{\mathcal{V}^T K_{\X,\X}^{-1} \mathcal{V} + 2Y^TK_{\X,\X}^{-1}\mathcal{V}},$$
and $\tilde{\Lambda}_x = \sum_{i=1}^m \sigma_v |[K_{\X,\X}^{-1} K_{\X,x}]_i|.$
\end{lemma}

We see that $\tilde{\Lambda}_x$ follows a similar form to the $\Lambda_x$ derived in Theorem \ref{my_det_theorem} (denoted as $\tilde{\Lambda}_x$ to make a visual comparison) however, since the inverse does not include the regularization term $\sigma_n^2I$ (i.e. $K_{\X,\X}^{-1}$ versus $(K_{\X,\X} + \sigma_n^2I)^{-1}$) then $\tilde{\Lambda}_x > \Lambda_x$ in general.
There is also an additional additive term $|Y^T (K_{\X,\X} + \sigma_n^{-2}K_{\X,\X}^2)^{-1}K_{\X,x}|$ in the bound derived by \cite{maddalena2021deterministic}, which introduces more conservatism when compared to our bound.
Hence, in general their deterministic error bound will be larger than our deterministic error bound.

Work \cite{maddalena2021deterministic} states that $\Delta - Y^TK_{X,X}^{-1}Y$ to can be set to 0 to avoid needing to compute $\Delta$. Similarly in our setting we can set $c^* = 0$. Note that if $\sigma_n = 0$ then $\sigma_D(x) = P(x)$. Our Theorem limit $\sigma_n > 0$ in order to guarantee the positive definiteness of $K_{X,X} + \sigma_n^2I$, however in the case that $K_{X,X}$ is strictly positive definite and invertible we can set $\sigma_n = 0$. Then the formulation for our deterministic error bound would be written as (using the same notation as the Equation~\eqref{maddalena_bound})
\begin{align}
    |\mu_{D}(x) - f(x)| \leq P(x)B + 
    \tilde{\Lambda}_x
\end{align}
whereas Equation~\eqref{maddalena_bound} still retains the additional additive term $|Y^T ( K_{\X,\X} + \frac{1}{m\sigma_n}K_{\X,\X}^2)^{-1}K_{\X,x}|$. Note that $\sigma_n$ cannot be 0 in this formulation and is not typically feasible for our formulation either. However, it is clear that our bound is tighter.

\section{Appendix: Proofs and Extended Discussion}
\subsection{Proof of Lemma \ref{lemma:3}}
\label{App:Lemma_proof}
\begin{proof}
Recall $G = (K_{\X,\X} + \sigma_n^2I)^{-1}$ and $W_x = K_{x,\X}G$.
Define in the RKHS the inner product $\{f,g\} = \langle f, g \rangle_{\kappa} - f(\X)^TGg(\X)$. Note that this remains positive definite as $f(\X)^TGf(\X) \leq \langle f, f \rangle_{\kappa}$ for any $G$ such that $K_{\X,\X} \leq G^{-1}$ (which is guaranteed for $\sigma_n > 0$) since $f(\X)^TK_{\X,\X}f(\X) \leq \langle f, f \rangle_{\kappa}$ by definition. Then
\begin{align}
    |\{f,\kappa(x,\cdot)\}| &= |\langle f, \kappa(x,\cdot) \rangle_{\kappa} - f(\X)^TG\kappa(\X,x)| \\
    &= |f(x) - W_xf(\X)|
\end{align}
and by the Cauchy-Schwarz inequality
\begin{align}
    |\{f,\kappa(x,\cdot)\}| &\leq \{\kappa(x,\cdot),\kappa(x,\cdot)\}^{1/2}\{f,f\}^{1/2} \\
    & \leq \sigma_D(x) \sqrt{\langle f, f \rangle_{\kappa} - f(\X)^TGf(\X)} \\
    & \leq \sigma_D(x) \sqrt{B^2 - c^*}
\end{align}
when $c^* \leq f(\X)^TGf(\X)$.
\end{proof}

\subsection{Proof of Theorem \ref{my_prob_theorem}}
\label{App:theorem_proof}
\begin{proof}
    This proof uses the error bound in Equation~\eqref{Eqn:InequalityProof}.  By applying Lemma~\ref{lemma:3} to the first term of Equation~\eqref{Eqn:InequalityProof}, we obtain
    \begin{align}
        |\mu_{D}(x) - f(x)| \leq \sigma_{D}(x)\sqrt{B^2 - c^*} + |W_x \V|. \label{err_bound_det}
    \end{align}
    To bound the second term $|W_x\V|$, we use Hoeffding's Inequality.
    \begin{lemma} [{\cite[Hoeffding's Inequality]{mohri2018foundations}}] \label{hoeffding_proof}
        Let $\px_1, \ldots, \px_m$ be independent random variables such that $a_i \leq \px_i \leq b_i$ almost surely. Consider the sum of these random variables $S_m = \sum_{i=1}^m \px_i$, then
        \begin{align}
            \label{eq:Hoeffdings Inequality Proof}
            \mathbb{P}(|S_m - \expect[S_m]| \geq t) \leq 2 \exp \left(\frac{-2t^2}{\sum_{i=1}^m(b_i-a_i)^2} \right). 
        \end{align} 
    \end{lemma}
    Using Lemma~\ref{hoeffding_proof}, we can reason about $|W_x \V|$ probabilistically as
    \begin{align}
        W_x \V = \sum_{i=1}^m [W_x]_i v_i
    \end{align}
    where $[W_x]_i$ denotes the $i$-th term of $W_x$. 
    Hence, the $S_m$ term in Equation~\eqref{eq:Hoeffdings Inequality Proof} becomes $W_x \V$ in our setting.
    Note that since each noise term $v_i$ has 0 mean, $\expect[W_x \V] = 0$. Also each $-\sigma_v \leq v_i \leq +\sigma_v$;
    hence, the denominator on the RHS of Equation~\eqref{eq:Hoeffdings Inequality Proof} becomes 
    \begin{align}
        % &\sum_{i=1}^m(\sigma_v [W_x]_i - -\sigma_v [W_x]_i)^2 = \sum_{i=1}^m(2\sigma_v [W_x]_i)^2\\
        &\sum_{i=1}^m(2\sigma_v [W_x]_i)^2 = 4\sigma_v^2W_xW_x^T = \lambda_x .
    \end{align}
    Then, by Hoeffding's Inequality, we have
    \begin{align}
        \mathbb{P}(|W_x \V| \geq t) \leq 2 \exp\left(-2t^2/ \lambda_x \right)
    \end{align}
    Taking the complement of the probability, we obtain
    \begin{align}
        \mathbb{P}(|W_x \V| \leq t) \geq 1 - 2 \exp\left(-2t^2 / \lambda_x \right). 
    \end{align}
    Denoting $\delta =2 \exp\left(\frac{-2t^2}{\lambda_x} \right)$, we see that $ t = \sqrt{\frac{\lambda_x}{2}\text{ln}(\frac{2}{\delta})}$. Hence
    \begin{align}
        \mathbb{P}\left(|W_x \V| \leq \sqrt{\frac{\lambda_x}{2}\text{ln}(\frac{2}{\delta})}\right) \geq 1 - \delta. \label{prob_WV}
    \end{align}
    
    Then by combining Equations \eqref{err_bound_det} and \eqref{prob_WV}, we conclude the proof.
\end{proof}

\subsection{On $\lambda_x$ Tends to 0 with Increased Data}
\label{app:disscuss}
Recall that $\lambda_x = 4\sigma_v^2K_{x,\X}G^2 K_{\X,x}$ and $G = (K_{\mathcal{X},\mathcal{X}} + \sigma_n^2I)^{-1}$.

Let $\lambda_i$ be the $i$-th eigenvalue of $K_{\mathcal{X},\mathcal{X}}$. The eigenvalues of $G^{-1} = K_{\mathcal{X},\mathcal{X}} + \sigma_n^2I$ are then $\lambda_i + \sigma_n^2$. The eigen-decomposition of $G^{-1}$ is then written as $U(\Lambda + \sigma_n^2I)U^T$ where $U$ is an orthogonal matrix of eigenvectors and $\Lambda = \text{diag}(\lambda_i)$. Similarly, $G = U(\Lambda + \sigma_n^2I)^{-1}U^T$ which has eigenvalues $1/(\lambda_i + \sigma_n^2)$. It is known that as data is added to $K_{\mathcal{X},\mathcal{X}}$, the complexity of the matrix tends to grow which results in the eigenvalues differentiating into a cluster of larger eigenvalues and a cluster of 0 eigenvalues.

We can write $K_{\mathcal{X},x} = U\alpha$, where $\alpha$ represents the projection onto the eigenspace, which then makes the quadratic form 
\begin{align}
K_{x,\mathcal{X}}GK_{\mathcal{X},x} = \alpha^T(\Lambda + \sigma_n^2I)^{-1}\alpha = \sum_{i=1}^{m} \frac{\alpha_i^2}{\lambda_i + \sigma_n^2}. \label{quad_form_expansion}
\end{align}
We also note that it is well known that $K_{x,\mathcal{X}}GK_{\mathcal{X},x}\leq \kappa(x,x)$ for any $\sigma_n > 0$ and any $m$ as the variance prediction of a GP cannot be negative. This directly implies that the projection of $K_{\X,x}$ onto the eigenvectors of $G$ have minimal to zero value when $\lambda_i$ is close to 0 (i.e. $\alpha_i$ is small or zero) as this is the only way in which the sum is guaranteed to have a finite upper bound. This is due to the 0 eigenvalues of $K_{\mathcal{X},\mathcal{X}}$ being associated with eigenvectors that represent redundancy in data (or random noise in high-dimensional systems), hence they have minimal impact on the value of the inverse quadratic form. 
Using $G^{2}$ in the quadratic form simply results in the eigenvalues being squared in the denominator of Eqn~\eqref{quad_form_expansion}, which reduces the impact of large eigenvalues as well. This results in the value of $K_{x,\mathcal{X}}G^2K_{\mathcal{X},x}$ tending towards 0 as more data is added. 

However, while the limit $\lambda_x \rightarrow 0$ is theoretically feasible, it is likely computationally intractable even for reasonable choices of $\sigma_n$ due to the conditioning of $G$ and the $\mathcal{O}(m^3)$ operation need to invert an $m \times m$ matrix. The term $\sigma_n$ can be considered as a regularization term that is designed to enable the invertability of $K_{\mathcal{X},\mathcal{X}} + \sigma_n^2I$ and also limits the upper bound of eigenvalues of $G$ to $1/(\sigma_n^2)$. This enables the mathematical formulation to remain sound despite the computational intractability.

\subsection{Proof of Corollary \ref{cor:uniform_bound}}
\label{App:cor_proof}
\begin{proof}
    Let $\overline{\epsilon}_{X'}(\delta) = \sup_{x\in X'}{\epsilon}(\delta,x)$, then $\epsilon(\delta,x) \leq \overline{\epsilon}_{X'}(\delta)$ for all $x\in X'$ by definition of supremum.
    When $\mathbb{P}\Big(|\mu_{D}(x) - f(x)| \leq \epsilon(x,\delta)\Big) \geq 1 - \delta$ holds, the probabilistic relation is still retained when $\epsilon(\delta,x)$ is substituted with any larger value. Hence $\overline{\epsilon}_{X'}(\delta)$ is a valid substitution for any $x \in X'$.
\end{proof}

\section{Appendix: Experiment Information}
\subsection{RKHS Norms and Data Domains}
\label{App:norms_domains}
We use results from \cite{hashimoto2022learning} to generate a bound on the RKHS norm of each system, we conservatively increase the bound determined through this method and validated our results through an empirical analysis of the bounds. In all cases our probabilistic bounds remained accurate for any $\delta$ chosen. We list the bounds on the norm and the domains over which this norm is valid in Table \ref{tab:rkhs_bounds}.

\begin{table*}[ht]
    \centering
    \begin{tabular}{c | c | c}
        \toprule
         System & $B_i$ & Domain\\
         \midrule
         1D Non-Linear & [40] & [0, 10]\\
         2D Linear & [10, 10] & $[-2, 2]\times[-2, 2]$\\
         2D Non-Linear & [10, 10] & $[-2, 2]\times[-2, 2]$\\
         3D Dubin's Car & [20, 5, 5] & $[0, 10]\times[0, 2]\times[-0.5, 0.5]$\\
         5D Second Order Car & [8, 8, 5, 5, 5] & $[-2, 2]\times[-2, 2]\times[-0.5, 0.5]\times[-0.5, 0.5]\times[-0.5, 0.5]$\\
         Inverted Pendulum & [2, 2] & $[-0.5, 0.5]\times[-1.1, 1.1]$ \\
         4D Linear & [1.4, 1.5, 1.4, 0.9] & $[-0.7, 0.7]\times[-0.7, 0.7]\times[-0.7, 0.7]\times[-0.7, 0.7]$
    \end{tabular}
    
    \caption{RKHS norm bounds used in experiments.}
    \label{tab:rkhs_bounds}
\end{table*}

\subsection{DKL Architectures}
\label{App:DKL_arch}

We provide the architectures for each network used in our DKL priors in Table \ref{tab:DKL_structures}. We note that for most systems, a single network could accurately predict the evolution of $f$ for all dimensions. However, the CartPole system have highly divergent dynamics for the velocities, requiring significant depth to accurately capture the evolution and enable our prior to remain accurate. The necessary network dimensions were determined by assessing when the error bounds remained valid with DKL (i.e. a well behaved network was learned). All models were trained via stochastic mini-batching (using 1/50-th of the data set each epoch) and all data was generated via a 1-time step evolution of the system from uniformly distributed points in the associated domain.

\begin{table*}[!ht]
    \centering
    \begin{tabular}{c | c | c | c | c | c}
        \toprule
         System & Activation & Hidden Layers & Layer Width & Training Data & Predictive Points\\
         \midrule
         2D Linear & GeLU & 2 & 16 & 1000 & 100\\
         2D Non-Linear & GeLU & 2 & 64 & 1000 & 200\\
         3D Dubin's Car & GeLU & 2 & 128 & 10000 & 400\\
         5D Second Order Car & GeLU & 2 & 128 & 10000 & 250 \\
         Inverted Pendulum ($\theta$) & GeLU & 1 & 32 & 50000 & 600 \\
         Inverted Pendulum ($\dot{\theta}$) & Tanh & 2 & 32 , 4 & 50000 & 600 \\
         4D Linear & GeLU & 1 & 16 & 20000 & 200
    \end{tabular}
    \caption{DKL neural network prior details. All DKL models use the exact same set of data for posterior predictions as the SE models (Predictive Points).}
    \label{tab:DKL_structures}
\end{table*}

\begin{table*}[!ht]
    \centering
    \begin{tabular}{c | c | c}
        \toprule
         System & SE min & DKL min\\
         \midrule
         2D Linear & [1.8,1.9] $\times 10^{-8}$ & [0.19, 0.12] \\
         2D Non-Linear & [1.9, 1.9]$\times 10^{-8}$ & [5.9, 7.8]$\times 10^{-2}$\\
         3D Dubin's Car & [0.0, 0.0, 0.0] & [0.03, 0.07, 41.4]$\times 10^{-2}$ \\
         5D Second Order Car & [2.6, 3.2, 6.7, 2.9, 1.6] $\times 10^{-15}$ & [8.8, 5.9, 11.0, 29.3, 22.2]$\times 10^{-2}$
    \end{tabular}
    \caption{Minimum $\kappa(x,x')$ across 50000 test points.}
    \label{tab:inf_compare}
\end{table*}

\subsection{Barrier Details}
\label{App:Barrier}

The 4D linear dynamics are defined below:
\begin{align}
    x(k+1) = \begin{bmatrix}
        0.6 & 0.2 & 0.1 & 0 \\
        0 & 0.63 & 0.2 & 0.05 \\
        0 & 0 & 0.51 & 0.15 \\
        0 & 0 & 0 & 0.2
    \end{bmatrix} x(k) + \pv
\end{align}
Note that this system is clearly contractive, hence emphasizing the conservatism induced by loose regression error bounds when synthesizing a barrier. 

We bound noise on the pole angle ($\theta$) and angular velocity ($\dot{\theta}$) as $\pm 0.1$ degree(/sec) for the Inverted Pendulum model. We train a different NN for each dimension due to the complexity introduced by the unknown controller to enable more accurate predictions for each dimension and tighter linear relaxations of the networks to calculate interval bounds on $f$ for barrier synthesis.  
The safe states (in radians(/sec)) for the pendulum are $[\theta, \dot{\theta}] \in [-0.5, 0.5]\times[-1.1, 1.1]$. For the 4D linear model the safe states are $x \in [-0.7, 0.7]$ for each dimension.

The 4D linear model has a bound on noise of $\pm0.05$ for each dimension, also using a PERT distribution. As this model was much simpler to learn, we use a single network to predict each dimension.

We use methods shown in Work \cite{reed2023promises} to generate interval bounds on the GP mean predictions and upper bound the variance prediction over discrete regions. These same methods are used to identify a upper bound on $\lambda_x$ and $\Lambda_x$ to generate the most conservative interval bounds. Work \cite{jackson2021formal} defines how to generate an interval MDP model from these bounds/intervals and a partitioning of the space Wwork \cite{mazouz2024piecewise} defines how to generate a barrier given either intervals on $f$ or an interval MDP model.

For the Pendulum environment, we generate a discretization with 1600 states, evenly discretizing each dimension into 40 sections. For the 4D linear model we generate a discretization of 2401 states, evenly discretizing each dimension into 7 sections. For each region we make use of Corollary \ref{cor:uniform_bound}, modifying $\bar{\epsilon}_{X'}$ to make use of Theorems \ref{my_prob_theorem}-\ref{my_det_theorem} and Lemma \ref{lemma:2}, to generate intervals on $f$ for Barrier synthesis.

With the noise bound already being fairly small for the pendulum model, we were required to set $\sigma_n > \sigma_v$ in order to allow for bounding the kernel outputs, this is particularly critical when bounding the output variance as the program is very sensitive to the condition number of the kernel matrix. This introduces some conservatism into the error bounds in \cite{hashimoto2022learning} as it requires assuming a larger support on noise, but causes no changes for our own bounds. For the 4D model we kept $\sigma_n = \sigma_v$.

\subsection{RKHS Norms for DKL}
\label{App:DKL_valid_norms}

Assuming a given norm for a squared exponential kernel, we would like to validate that this norm will remain valid for DKL. We can assess this by validating that $\inf_{x,x'\in X}\kappa_{DKL}(x,x') \geq \inf_{x,x'\in X}\kappa(x,x')$. We empirically validate this inequality over 50000 random test points in our domain with results in Table \ref{tab:inf_compare}, in all cases the DKL model we generate has larger values.

\end{document}